\documentclass{article}

 \usepackage[preprint]{neurips_2026}
 \usepackage{caption}

\usepackage[utf8]{inputenc} 
\usepackage[T1]{fontenc}    
\usepackage{hyperref}       
\usepackage{url}            
\usepackage{booktabs}       
\usepackage{amsfonts}       
\usepackage{nicefrac}       
\usepackage{microtype}      

\usepackage[colorinlistoftodos,prependcaption,textsize=tiny]{todonotes}
\PassOptionsToPackage{numbers,sort&compress}{natbib}
\usepackage{amsmath}
\usepackage{algorithm}
\usepackage{algorithmic}
\usepackage{wrapfig}
\usepackage{tcolorbox}
\usepackage{graphicx}
\usepackage{wrapfig}
\usepackage{mdframed}
\usepackage{amsthm}
\usepackage{amsfonts}

\newtheorem{theorem}{Theorem}

\newtheorem{lemma}{Lemma}

\usepackage{siunitx}
\usepackage{makecell}
\usepackage{fancyvrb}
\usepackage{multirow}
\usepackage{booktabs}
\usepackage[table,dvipsnames]{xcolor} 
\usepackage{hyperref}
\usepackage{enumitem}

\hypersetup{colorlinks=true,linkcolor=MidnightBlue,anchorcolor=black,citecolor=MidnightBlue,urlcolor=MidnightBlue}

\title{Learn Where Outcomes Diverge: Efficient VLA RL via Probabilistic Chunk Masking}

\author{Vaidehi Bagaria$^{*}$, Nikshep Grampurohit$^{*}$, Pulkit Verma$^{\dagger}$ \\ Indian Institute of Technology Madras}

\begin{document}

\maketitle

\renewcommand{\thefootnote}{\fnsymbol{footnote}}
\footnotetext[1]{Equal contribution.}
\footnotetext[2]{Correspondence to \texttt{pulkitv@cse.iitm.ac.in}.}
\renewcommand{\thefootnote}{\arabic{footnote}}

\begin{abstract}
Reinforcement learning (RL) allows vision-language-action (VLA) policies to generalize beyond their training distribution by optimizing directly for task success, but this post-training stage is computationally expensive. A natural response has been to speed up rollout data collection through faster simulators and world models. In
GRPO-based VLA RL, however, we find that the dominant cost lies elsewhere:
gradient computation accounts for $\sim$78\% of wall-clock time per
step in our runs, while rollout collection accounts for only
$\sim$21\%. Gradient cost dominates because much of this computation is spent on phases that contribute little to learning. GRPO's learning signal is driven by advantage variance: only phases where successful and failed rollouts diverge produce useful learning signal. However, GRPO assigns the same advantage to every chunk in a rollout. As a result,
actor-update compute is spent uniformly across the trajectory, including phases the policy already handles after pre-training and supervised fine-tuning. This paper presents Probabilistic Chunk Masking (PCM), a drop-in
modification to GRPO that allocates gradient computation to a small,
probabilistically selected subset of chunks per trajectory. PCM scores
semantic phases using success-failure action variance, a rollout-derived
proxy for per-phase gradient variance, and samples a fixed chunk budget
with online-updated phase-level keep probabilities. We formalize per-phase gradient variance as the quantity that determines where gradient computation is useful and show that success-failure action
variance provides a measurable proxy for it.
PCM requires no reward model or learned critic. On three LIBERO benchmarks, PCM matches
the final success rate of standard GRPO while achieving $2.38{\times}$
wall-clock speedup, $4.8{\times}$ faster gradient updates, and 60$\%$
lower peak activation memory, while backpropagating through fewer than 20$\%$ of trajectory
chunks.

\end{abstract}

\section{Introduction}

Reinforcement learning (RL) has become essential for vision-language-action (VLA) models, enabling generalization to new environments and behaviors~\citep{black2026pi0visionlanguageactionflowmodel, kim2024openvlaopensourcevisionlanguageactionmodel}. The field has increasingly converged on Group Relative Policy Optimization (GRPO)~\citep{shao2024deepseekmathpushinglimitsmathematical} whose critic-free formulation removes the cost and complexity of training a separate value model~\citep{li2025simplevlarlscalingvlatraining, lu2025vlarlmasterfulgeneralrobotic}. As these systems scale, speeding up post-training has become critical. Efforts to speed up GRPO-based VLA RL have implicitly assumed that rollout collection is the bottleneck. We measure where time actually goes and find the reverse: in our runs gradient computation accounts for $\sim$78\% of wall-clock time per 
training step, while rollout collection accounts for only $\sim$21\% 
(Fig.~\ref{fig:phase_time}, right). 


This inefficiency arises from a deeper structural feature shared by most RL algorithms. Backpropagation runs over every chunk in the trajectory, and trajectory-level credit assigns the same advantage to each chunk regardless of which determined the outcome. Both inefficiencies share a common cause: the trajectory has been treated as the unit of gradient computation. Even methods that differentiate per-step credit, such as PPO~\citep{schulman2017proximalpolicyoptimizationalgorithms} with a learned critic, still apply forward/backward-pass compute to all sampled timesteps. They modify the signal, not the compute allocation. 

This observation suggests a missing axis in VLA RL optimization: \textit{within a trajectory, which chunks should receive gradient computation at all?}

The answer is not all of them. GRPO’s learning signal is driven by differences in advantage across rollouts: when rollout rewards are uniform, gradients vanish 
\citep{wang2026treestylebranchingmattersthought, 
xu2025singlestreampolicyoptimization}. We argue that variance non-uniformity persists even within a single non-collapsed group. Supervised fine-tuning is effective in densely covered, unimodal regions,
but weak under sparse demonstrations or multimodal optimal behavior~\citep{chu2025sftmemorizesrlgeneralizes}; these
are precisely the regimes that RL fine-tuning is intended to refine. 
This suggests that (a) gradient computation should be treated as an allocatable resource and (b) recomputing gradients on phases the base policy already handles wastes
compute and memory. From this perspective, RL training becomes a problem of allocating limited gradient budget to the parts of trajectories that carry useful learning signal.

We present Probabilistic Chunk Masking (PCM), a drop-in
modification to GRPO that updates only a fixed budget of chunks per
trajectory. PCM assigns phase-level keep probabilities using the
success-failure action variance $C_c$, which is concentrated in a small
number of outcome-critical semantic phases $c$ (Fig.~\ref{fig:phase_time}, left). It then physically removes the unselected chunks from the backward pass, significantly reducing
actor-update compute and activation memory. We further prove that $C_c$ serves as a measurable lower-bound proxy for the true per-phase gradient variance $V_c$, giving a principled
basis for this allocation. The 
\textbf{main contributions} of this paper are as follows:

\begin{figure}[t]
\centering

\begin{minipage}{0.56\linewidth}
\centering
\includegraphics[width=\linewidth]{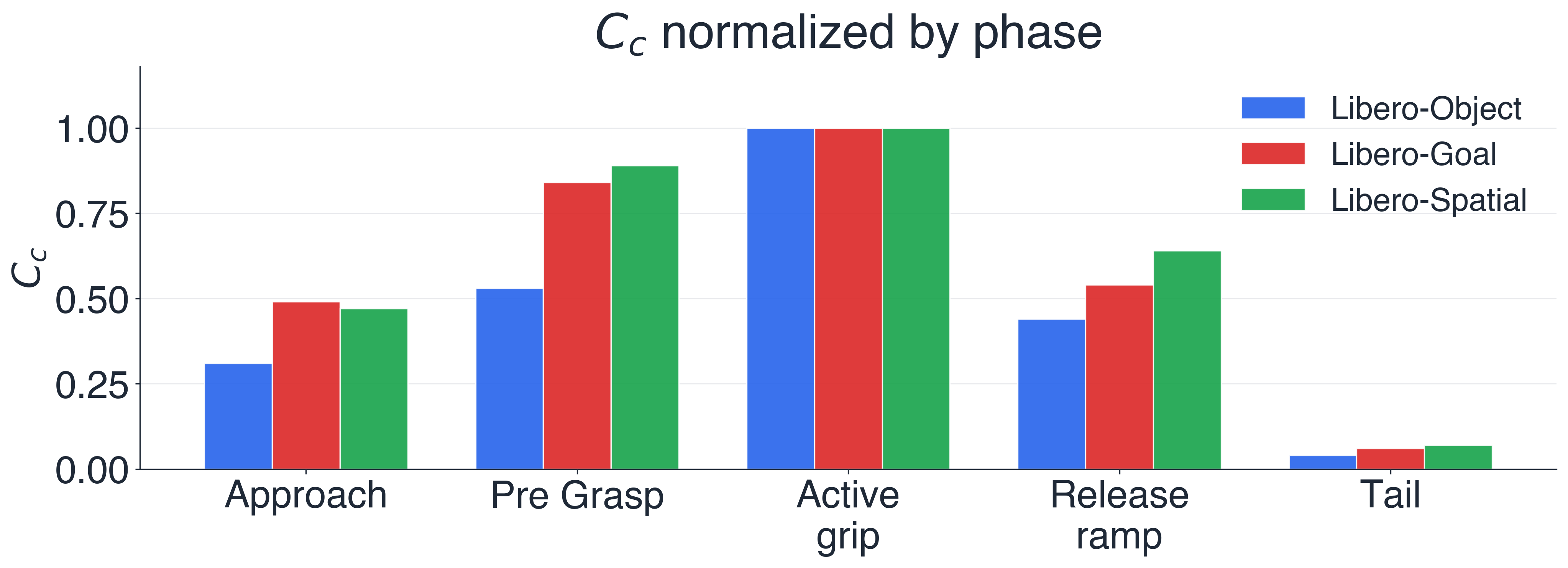}
\end{minipage}
\hfill
\begin{minipage}{0.4\linewidth}
\centering
\rowcolors{2}{gray!13}{white}
\small
\setlength{\tabcolsep}{2pt}
\begin{tabular}{lcc}
\toprule
\textbf{Benchmark} & \textbf{Grad} \textbf{(\%)} & \textbf{Rollout} \textbf{(\%)} \\
\midrule
LIBERO-Object  & 77.91 & 21.27 \\
LIBERO-Goal    & 77.86 & 21.29 \\
LIBERO-Spatial & 77.85 & 21.29 \\
\bottomrule
\end{tabular}
\end{minipage}
\caption{Left: per-phase success--failure action variance ($C_c$)
concentrates on decision-critical phases. Right: gradient computation
dominates wall-clock training cost in simulator-based GRPO with
10 rollouts per prompt. Remaining $\sim$1\% is logging, evaluation, and inter-step overhead.}
\label{fig:phase_time}
\vspace{-15pt}
\end{figure}


\textbullet\ \,
\textbf{We show that $V_c$ is the right signal for gradient allocation}
and derive a Neyman allocation proportional to $N_c\sqrt{V_c}$ ($N_c$ is the number of chunks in phase $c$), which
concentrates gradient compute on outcome-divergent phases.

\textbullet\ \,
\textbf{We establish that the per-phase success–failure action variance $C_c$ is a sufficient proxy for $V_c$}, computable directly from rollouts GRPO already produces with no reward model or learned critic.

\textbullet\ \,\textbf{Empirically, these enable substantial efficiency gains in gradient-bound VLA-RL.} 
On LIBERO~\citep{liu2023liberobenchmarkingknowledgetransfer}, PCM achieves 
$2.38\times$ wall-clock speedup to 98\% convergence, $4.8\times$ faster gradient steps, 60\% 
lower peak activation memory, using less than 20\% of chunks per trajectory.

\section{Related works}

\paragraph{Reinforcement Learning Fine-Tuning of VLA Models.}
Reinforcement learning has become an effective post-training stage for vision-language-action (VLA) models~\citep{chu2025sftmemorizesrlgeneralizes, liu2026rlbringvlageneralization}. Early approaches relied on learned critics~\citep{hu2024flareachievingmasterfuladaptive, guo2025improvingvisionlanguageactionmodelonline, chen2025conrftreinforcedfinetuningmethod, tan2025interactiveposttrainingvisionlanguageactionmodels}, but the field has converged on critic-free GRPO-style optimization as these systems have scaled ~\citep{li2025simplevlarlscalingvlatraining,lu2025vlarlmasterfulgeneralrobotic,chen2025tgrpofinetuningvisionlanguageactionmodel,ye2025vlar1enhancingreasoningvisionlanguageaction}.

To support this shift, advances in world models and simulation have significantly reduced rollout cost ~\citep{zhu2025wmpoworldmodelbasedpolicy, liu2026worldvlaloopclosedlooplearningvideo, zang2026rlinfvlaunifiedefficientframework} under the implicit assumption that data collection is the bottleneck. However, this progress has left gradient computation, which we observe accounts for $\sim$78\% of training time, largely unoptimized.
Other work improves the source or structure of the reward signal~\citep{li2025vlarftvisionlanguageactionreinforcementfinetuning,zhai2025visionlanguageactioncriticmodelroboticrealworld,tan2025interactiveposttrainingvisionlanguageactionmodels} by altering the reward, critic, rollout source, or advantage/credit-assignment rule. We show that this additional machinery can be avoided: $C_c$ (success-failure action variance within a phase), an internal signal already present in GRPO rollouts, directly captures where learning signal concentrates and is sufficient to guide gradient allocation. 

\paragraph{Efficiency Methods for GRPO.}
GRPO’s learning signal is driven by advantage variance~\citep{nan2025ngrponegativeenhancedgrouprelative}, and prior large language model (LLM) RL work has exploited this at two granularities. At the prompt level, methods filter or reweight zero-variance rollout groups~\citep{yu2025dapoopensourcellmreinforcement,zheng2025actpaysefficientreinforcement,hu2025vadevarianceawaredynamicsampling}. At the token level, work selects which sequence positions receive gradient using policy-internal signals such as token entropy or probability shifts ~\citep{wang20258020rulehighentropyminority,khandoga2026uniformcreditcausalcredit, hu2026entropygatedselectivepolicyoptimizationtokenlevel}. PCM operates at a third granularity: phase-level allocation within a trajectory. Entropy and probability shifts are indirect proxies that measure where the model is uncertain, not where successful and failed rollouts actually diverge. In settings with verifiable outcomes (like VLA RL), this raises a more fundamental question, \textit{why estimate outcome-criticality indirectly when the outcome signal is already available?} We use $C_c$, computed directly from rollout outcomes, to allocate gradient compute according to observed success–failure divergence. Unlike prior work that relies on proxy signals such as entropy or probability shifts, PCM exploits outcome-grounded divergence.

Parameter-efficient methods like LoRA and its extensions are
standard in LLM and VLA fine-tuning, reducing trainable parameters but
not the activation cost of backpropagating through long trajectories
\citep{hu2021loralowrankadaptationlarge,
aghajanyan2020intrinsicdimensionalityexplainseffectiveness,
dettmers2023qloraefficientfinetuningquantized,
zhang2025trainsmallinferlarge,
zhang2023adaloraadaptivebudgetallocation}. Activation-focused methods reduce this memory cost through checkpointing and streaming backpropagation ~\citep{chen2016training, korthikanti2022reducing, luo2025streambpmemoryefficientexactbackpropagation}. Crucially, prior methods do not address the inefficiency of computing gradients on outcome-invariant phases within a trajectory, which still incur computation cost despite contributing little to learning.

\section{Preliminaries}
\label{sec:prelim}
We briefly describe the setup for chunk-level GRPO, which our method builds on.
 
\paragraph{VLA Policies and Action Chunks.}
Let $\pi_\theta$ denote a VLA policy that produces actions in chunks of length
$L$ conditioned on observation $s_{i,k}$, where $i$ indexes the trajectory and
$k$ the chunk within that trajectory.
A rollout of trajectory $i$ decomposes into $N_i = \lceil T_i / L \rceil$
chunks, where $T_i$ is the total number of timesteps.

\paragraph{Group Relative Policy Optimization.}
GRPO~\citep{shao2024deepseekmathpushinglimitsmathematical} samples a group of $G$ rollouts from the
current policy, scores each with a binary reward $r_i \in \{0, 1\}$, and forms
the group-relative advantage
\begin{equation}
  A_i = \frac{r_i - \mu_r}{\sigma_r + \varepsilon},
  \qquad
  \mu_r = \frac{1}{G}\sum_j r_j,
  \quad
  \sigma_r^2 = \frac{1}{G}\sum_j (r_j - \mu_r)^2.
  \label{eq:advantage}
\end{equation}
Writing $a_{i,k}$ for the action tokens of chunk $k$ in trajectory $i$, the
chunk-level GRPO objective is
\begin{equation}
  \mathcal{L}_{\mathrm{GRPO}}(\theta)
  = -\mathbb{E}_i\!\left[
      \sum_{k=1}^{N_i} A_i \cdot \log\pi_\theta(a_{i,k} \mid s_{i,k})
    \right],
  \label{eq:grpo}
\end{equation}
with PPO-style clipping omitted for clarity. 

\section{Methodology: Probabilistic Chunk Masking}
\label{sec:method}

\begin{figure}
    \centering
    \includegraphics[width=0.9\linewidth]{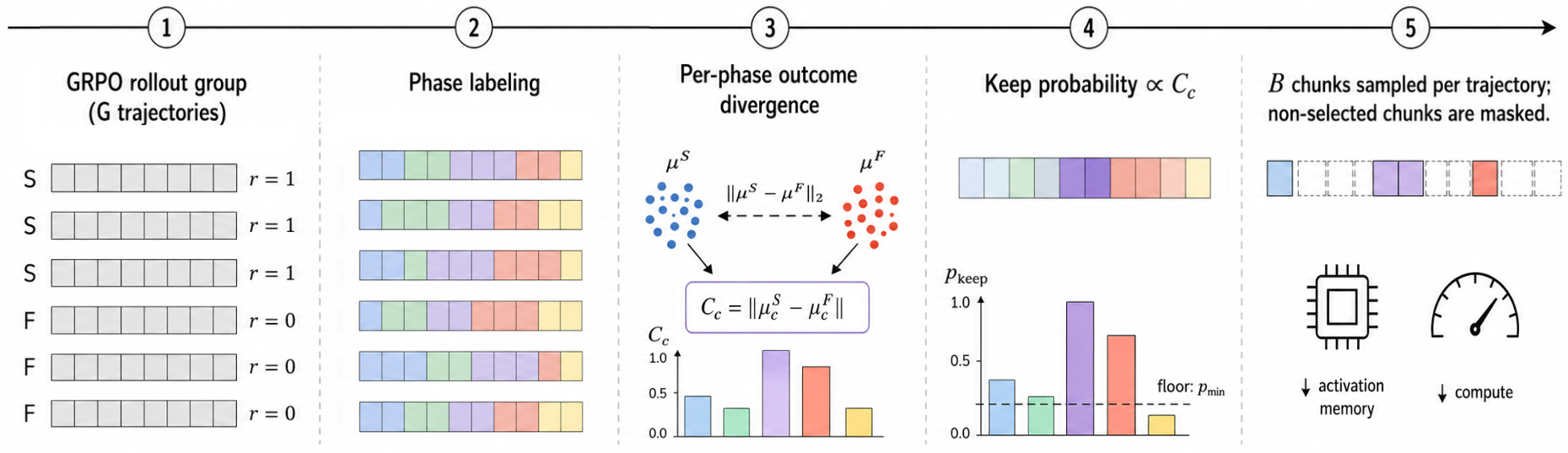}
    \caption{Probabilistic Chunk Masking (PCM). We compute per-phase success–failure action variance ($C_c$) from GRPO rollouts, use it to prioritize decision-critical phases, and sample a fixed budget of chunks per trajectory. Gradients are computed only on selected chunks.
}
    \label{fig:methodology}
\end{figure}


\subsection{Phase-Conditioned Gradient Variance}
\label{sec:variance}
 
Following the VLA-RL setup formalized in Appendix~\ref{app:problem}, we partition every trajectory into $K$ semantic phases using a deterministic
labeling rule (described in Sec.~5.1).
Let $\mathcal{P} = \{c_1, \ldots, c_K\}$ denote this phase set, and let
$\varphi(i,k) \in \mathcal{P}$ denote the phase of chunk $(i,k)$. 
 
\paragraph{Phase decomposition of the GRPO gradient.}
The gradient of $\mathcal{L}_{\mathrm{GRPO}}$ decomposes by phase:
\begin{equation}
  \nabla_\theta \mathcal{L}_{\mathrm{GRPO}}
  = \sum_{c \in \mathcal{P}} g_c(\theta),
  \qquad
  g_c(\theta)
  = -\mathbb{E}_i\!\left[
      A_i \sum_{\substack{k \,:\, \varphi(i,k)=c}} \nabla_\theta \log\pi_\theta(a_{i,k} \mid s_{i,k})
    \right].
  \label{eq:decomp}
\end{equation}
The learning signal available in phase $c$ is characterized by the
\emph{per-phase gradient variance}
\begin{equation}
  V_c = \mathrm{Var}\!\left(
    A_i \cdot \nabla_\theta \log\pi_\theta(a_{i,k} \mid s_{i,k})
    \;\middle|\;
    \varphi(i,k) = c
  \right).
  \label{eq:Vc}
\end{equation}
Intuitively, $V_c$ is large for phases where the policy behaves differently
across successful and failed rollouts.
Throughout, $V_c$ refers to the within-trajectory chunk variance ($A_i$ fixed); under advantage normalization (Eq.~\ref{eq:advantage}), this matches the unconditional variance in Eq. (4) up to an $O(1)$ constant.
We formalize this observation as follows. Note that formal proofs appear in Appendix~\ref{sec:formal_proofs}.
 
\begin{lemma}[Phase gradient variance]
\label{lem:variance}
Let $c \in \mathcal{P}$ be a phase in which $\pi_\theta$ has converged, i.e.,
the action distributions of successful and failed rollouts are identical in
phase $c$.
Then $\|g_c(\theta)\| \approx 0$ and $V_c \approx 0$, so gradient
samples from phase $c$ contribute negligible signal.
Conversely, for a phase $c$ in which the base policy is underspecified,
$V_c$ is large and gradient samples from phase $c$ are informative.
\end{lemma}

 
\paragraph{Empirical proxy for $\mathbf{V_c}$.}
Computing $V_c$ exactly requires access to the gradient distribution, which is
expensive.
We instead use the \emph{success-failure action variance}
\begin{equation}
  C_c
  = \bigl\|
      \mathbb{E}[a_{i,k} \mid r_i = 1,\, \varphi(i,k) = c]
      -
      \mathbb{E}[a_{i,k} \mid r_i = 0,\, \varphi(i,k) = c]
    \bigr\|,
  \label{eq:Cc}
\end{equation}
computable directly from the rollout group that GRPO already produces, with no
auxiliary model or annotation. We compute $C_c$ only for rollout groups with non-zero reward variance.
 
\begin{lemma}[$C_c$ as a proxy for $V_c$]
\label{lem:proxy}
For a policy $\pi_\theta$ that is locally Gaussian with per-dimension variance
$\sigma_\pi^2$, the per-phase gradient variance satisfies
  $V_c \;\geq\; {C_c^2}/{4\sigma_\pi^2}$.
\end{lemma}
 

Note that the bound in Lemma~\ref{lem:proxy} is a lower bound that absorbs $\|\nabla_\theta \mu_\theta(s)\|^2$ into the constant. It justifies using $C_c$ to recover the \emph{ordering} of $V_c$ across phases, 
rather than exact correspondence. We treat $C_c$ as a practical proxy supported by empirical evidence: the realized allocation (Fig.~\ref{fig:grad_prog}) closely matches the $\sqrt{V_c}$-weighted prediction, and variance-aware ablations (Fig.~5b) show that $C_c$-weighted selection outperforms both random masking and highest-variance-only selection.


 
 
\subsection{Optimal Budget Allocation}
\label{sec:allocation}
 
We now state the core theoretical result.
Given a fixed chunk budget $B$ per trajectory and phase variance estimates
$\{V_c\}_{c \in \mathcal{P}}$, we ask: how should $B$ be distributed across
phases to minimize the variance of the GRPO gradient estimator (the sampling noise around the true gradient, distinct from the per-phase signal $V_c$)? Lower estimator variance means a more accurate estimate of the true gradient per step, which translates directly to faster stochastic gradient descent (SGD) convergence~\citep{bottou2018optimization}. This follows Neyman-allocation intuition:
under a fixed sampling budget, phases should be sampled in proportion to
their contribution to estimator variance.
 
\begin{theorem}[Optimal phase allocation]
\label{thm:main}
Let $b_c$ denote the number of chunks sampled from 
phase $c$, with $\sum_{c \in \mathcal{P}} b_c = B$, 
and let $N_c$ denote the expected number of chunks 
in phase $c$ per trajectory. The allocation minimizing 
the variance of the unbiased GRPO gradient estimator 
(Eq.~(\ref{eq:grpo}) subject to the budget constraint is
\begin{equation}
    b_c^* = B \cdot \frac{N_c\sqrt{V_c}}
    {\sum_{c' \in \mathcal{P}} N_{c'}\sqrt{V_{c'}}}.
    \label{eq:optimal-alloc-main}
\end{equation}
\end{theorem}

The speedup over uniform allocation depends on how 
sharply $V_c$ concentrates across phases. When $V_c$ 
is uniform across all $K$ phases, Eq.~\eqref{eq:optimal-alloc-main} 
reduces to uniform allocation $b_c = B/K$. When $V_c$ 
concentrates on a small subset of phases, the Neyman 
allocation assigns disproportionately more budget to 
those outcome divergent phases, reducing estimator variance relative to 
uniform at the same total budget $B$. The concentration 
of $C_c$ on contact-rich phases observed in Fig.~\ref{fig:phase_time}, 
combined with Lemma~2, implies that $V_c$ is similarly 
concentrated, making this regime directly applicable 
here. Section~\ref{sec:exp} quantifies the resulting wall-clock 
speedup empirically. A formal convergence-rate 
derivation and speedup formula are provided in 
Appendix~\ref{app:allocation}.

\subsection{Online Phase Score Estimation}
\label{sec:score}
 
Theorem~\ref{thm:main} requires knowledge of $\{V_c\}$, which is not directly
observable.
By Lemma~\ref{lem:proxy}, the relative ordering of $\{V_c\}$ is captured by the rollout-computable proxy
$\{C_c\}$. 
We therefore use $C_c$ to set per-phase weights; since phase $c$ contains $N_c$
candidate chunks, 
weighted sampling consistent with 
Eq.~\eqref{eq:optimal-alloc-main} gives 
$\mathbb{E}[b_c]\propto N_c\sqrt{V_c}$, concentrating 
gradient computation on outcome-divergent phases. 
Phases with larger estimated $V_c$ receive higher 
keep probability, while
larger $N_c$ contributes more candidate chunks. Both factors increase
sampling frequency within the fixed budget $B$.

 
During training, we compute a batch-level phase score $C_c^{(t)}$ from
the rollout group at step $t$. We maintain a short buffer
$\mathcal{B}_c$ for each phase and append the current score to the short refresh buffer. With $G = 10$ rollouts and binary rewards, individual $C_c$ estimates may be noisy, especially under imbalanced success-failure splits. Therefore, scores are averaged over $T_{\mathrm{rc}}=5$ batches before
refreshing the keep probabilities.
\begin{equation}
  \mathcal{B}_c \leftarrow \mathcal{B}_c \cup \{C_c^{(t)}\}.
  \label{eq:buffer}
\end{equation}
Every $T_{\mathrm{rc}}$ steps (we use $T_{\mathrm{rc}}=5$), we collapse
the buffered scores into a share-normalized phase score:
\begin{equation}
  S_c = \sum_{C_c^{(t)} \in \mathcal{B}_c} C_c^{(t)},
  \qquad
  \rho_c = \frac{S_c}{\sum_{c'} S_{c'}},
  \qquad
  \tilde{\rho}_c =
  \frac{\rho_c}{\max_{c'} \rho_{c'}} \in [0,1].
  \label{eq:score}
\end{equation}
Share normalization is scale-invariant: it preserves the relative phase
ordering even though the absolute magnitude of $C_c$ can vary across
tasks and training stages. The resulting $\tilde{\rho}_c$ is used as the
phase-level score for probabilistic chunk selection.
For the first
update, the keep probabilities are computed from the phase scores of the
first rollout group itself; after that, scores are updated every
$T_{\mathrm{rc}}$ steps using the buffered window. This online update lets \textsc{PCM} track the realized learning signal:
if a phase learns faster, its success--failure divergence decreases and keep probability naturally falls.

\subsection{Probabilistic Mask Selection}
\label{sec:mask}
 
Given phase scores $\{\tilde{\rho}_c\}$ as proxies for $\{\sqrt{V_c}\}$, we
map them to per-phase keep-probabilities that implement the allocation of
Theorem~\ref{thm:main} in expectation:
\begin{equation}
  p_c = \max\!\left(p_{\min},\; \tilde{\rho}_c\right) \in [p_{\min}, 1],
  \label{eq:keepprob}
\end{equation}
with $p_{\min} = 0.1$.
The floor $p_{\min}$ prevents any phase from being structurally excluded
between buffer recomputes.
This is important because $\tilde{\rho}_c$ may be transiently zero for a
phase with nonzero true variance $V_c$ if that phase is underrepresented or temporarily underestimated in
the current buffer window.
Excluding such phases entirely would violate the exploration requirement
of the allocation and could cause the estimated $\{\rho_c\}$ to diverge from
the true $\{V_c\}$ ordering over time.
Each chunk inherits its phase's probability weight: $w_{i,k} = p_{\varphi(i,k)}$.
The set $\{p_c\}$ is held fixed across the next $T_{\mathrm{rc}}$ steps.

\subsection{Fixed-Budget Sampling and Physical Shrinking}
\label{sec:shrink}
Given the allocation rule derived in Theorem~\ref{thm:main}, we sample a fixed budget of B chunks per trajectory using weighted sampling without replacement ~\citep{plackett1975analysis,luce1959individual}, with weights
$\{w_{i,k}\}$:
\begin{equation}
  \mathcal{K}_i \sim \mathrm{WeightedSample}\!\left(
    w_{i,1}, \ldots, w_{i,N_i};\;
    \min(B, N_i),\;
    \text{w/o replacement}
  \right).
  \label{eq:sample}
\end{equation}

The selected set is $\mathcal{K}_i=\{k_1,\ldots,k_{m_i}\}$. Non-selected chunks are physically removed from the batch tensor before the forward pass. 

The actor update uses the masked objective:
\begin{equation}
\mathcal{L}_{\mathrm{PCM}}(\theta) = 
-\mathbb{E}_i\!\left[
    \sum_{k \in \mathcal{K}_i}
    A_i \cdot \log \pi_\theta(a_{i,k} \mid s_{i,k})
\right]
  \label{eq:pcm}
\end{equation}

The masked objective in Eq.~(\ref{eq:pcm}) omits the $1/p_c$ importance weights, making it a biased estimator of the full-trajectory GRPO gradient. The bias is $\sum_{c \in \mathcal{P}} (1 - p_c)\, g_c$: chunks that are not selected contribute zero to Eq.~(\ref{eq:pcm}) but $g_c$ to the full gradient. Under our allocation $p_c \propto \sqrt{V_c}$, the factor $(1-p_c)$ is large precisely when $V_c$ is small, and by Lemma~1, small $V_c$ implies small $\|g_c\|$. Each term in the bias is therefore suppressed by at least one small factor, so Theorem~1's allocation remains approximately optimal under the biased estimator. The trade-off is favorable: at the same chunk budget $B$, the masked estimator has substantially lower variance than the importance-weighted unbiased alternative, translating to faster SGD convergence (Appendix~\ref{app:allocation}). A formal bound is given in Appendix~\ref{app:bias}.

\begin{algorithm}[!h]
\caption{Probabilistic Chunk Masking for GRPO}
\label{alg:pcm}
\begin{algorithmic}[1]
\REQUIRE Policy $\pi_\theta$; group size $G$; chunk budget $B$; floor $p_{\min}$; refresh window $T_{\mathrm{rc}}$
\STATE Initialize phase-score buffers $\mathcal{B}_c\leftarrow[\,]$ for all phases $c$
\WHILE{not converged}
    \STATE Sample $G$ rollouts and compute binary rewards $\{r_i\}$ \hfill $\triangleright$ Eq.~\eqref{eq:advantage}
    \STATE Compute GRPO advantages $\{A_i\}$ from the group rewards \hfill $\triangleright$ Eq.~\eqref{eq:advantage}
    \STATE Label each chunk with a phase $\varphi(i,k)$ from the gripper trajectory \hfill $\triangleright$ Sec.~\ref{sec:exp-settings}
    \STATE Compute per-phase success-failure action variance
    $\{C_c^{(t)}\}_{c\in\mathcal{P}}$ from rollouts
    \hfill $\triangleright$ Eq.~\eqref{eq:Cc}
    
    \STATE Append $C_c^{(t)}$ to buffer $\mathcal{B}_c$ for all $c \in \mathcal{P}$
    \hfill $\triangleright$ Eq.~\eqref{eq:buffer}
    \IF{first batch \textbf{or} $|\mathcal{B}_c|=T_{\mathrm{rc}}$}
        \STATE Collapse buffers, share-normalize and max-normalize 
        \hfill $\triangleright$ Eq.~\eqref{eq:score}
        \STATE Apply floor:
        $p_c \leftarrow \max(p_{\min},\,\tilde\rho_c)$ for all $c$
        \hfill $\triangleright$ Eq.~\eqref{eq:keepprob}
        
        \STATE Reset phase buffers:     $\mathcal{B}_c\leftarrow[\,]$ for all $c\in\mathcal{P}$
    \ENDIF
    \FOR{each trajectory $i$ with $N_i$ chunks}
        \STATE Set chunk weights $w_{i,k} \leftarrow p_{\varphi(i,k)}$ for $k=1,\ldots,N_i$ \hfill $\triangleright$ Sec.~\eqref{sec:mask}
        \STATE Sample $\mathcal{K}_i$ with $|\mathcal{K}_i|=\min(B,N_i)$ via weighted sampling w/o replacement \hfill $\triangleright$ Eq.~\eqref{eq:sample}
        \STATE Physically remove non-selected chunks from batch; keep only $\mathcal{K}_i$ for actor update \hfill $\triangleright$ Sec.~\ref{sec:shrink}
    \ENDFOR
    \STATE Update $\pi_\theta$ using the masked GRPO loss \hfill $\triangleright$ Eq.~\eqref{eq:pcm}
\ENDWHILE
\end{algorithmic}
\end{algorithm}

\section{Empirical Evaluation}
\label{sec:exp}

\subsection{Experimental Setup}
\label{sec:exp-settings}


\paragraph{Models and Benchmarks.}
We run our experiments on OpenVLA-OFT \citep{kim2025finetuningvisionlanguageactionmodelsoptimizing}, a 7B vision-language-action model that predicts 7-DoF action chunks with chunk length $L{=}8$. We 
evaluate on three benchmarks: LIBERO-Object, LIBERO-Spatial, and LIBERO-Goal~\citep{liu2023liberobenchmarkingknowledgetransfer}. 

\paragraph{Phase Labeling.}
We assign each chunk to one of $K{=}5$ semantic phases using a deterministic, 
multi-grasp-aware rule over the per-chunk gripper-close fraction 
$g_f[j] \in [0,1]$. Active-grip chunks ($g_f[j] \geq 0.5$) capture 
sustained grasp and transport; pre-grasp labels the up-to-three chunks 
immediately before sustained closure ($0.1 \leq g_f[j] < 0.5$); release-ramp 
labels the up-to-three chunks after; approach covers the remaining 
pre-contact chunks; and tail covers post-release open-gripper chunks; 
Overlapping windows are resolved by the priority ordering 
$\texttt{active-grip} > \texttt{pre-grasp} > \texttt{release-ramp} > 
\texttt{approach} > \texttt{tail}$.
This rule is applied identically to successful and failed trajectories, 
enabling $C_c$ to compare outcomes under a shared phase partition. 
Full details are in Appendix~\ref{app:prompts}.

\paragraph{Training and Evaluation.}
We implement PCM on the SimpleVLA-RL~\citep{li2025simplevlarlscalingvlatraining} verl~\citep{10.1145/3689031.3696075} pipeline with GRPO groups of 10 rollouts per prompt, comparing against the full-trajectory chunk-level GRPO baseline. We fine-tune OpenVLA-OFT with LoRA using 2 NVIDIA H100 GPUs. We evaluate every training step using 50 held-out validation
rollouts. Results are averaged across three seeds, and we report per-step update
time and peak GPU memory. Detailed experimental settings are in Appendix~\ref{app:setup}. 

\paragraph{Research Questions.}
Our experiments are designed to answer the following:
\begin{itemize}
    \item[\textbf{RQ1}] Does variance-aware chunk selection preserve the learning dynamics of full-trajectory GRPO, measured by step-wise accuracy curves?
    \item[\textbf{RQ2}] Does selective gradient allocation reduce wall-clock time to reach a target success rate without affecting sample efficiency?
    \item[\textbf{RQ3}] How sensitive is performance to the chunk budget $B$, and is there a stable operating point that balances gradient signal with per-step speedup?
    \item[\textbf{RQ4}] Is variance-aware probabilistic selection necessary, or do simpler masking strategies achieve comparable performance under the same chunk budget?
\end{itemize}

\subsection{Results and Discussion}
\label{sec:results}

\paragraph{RQ1: Learning Dynamics.}
\begin{figure}[!h]
    \centering
    \includegraphics[width=1\linewidth]{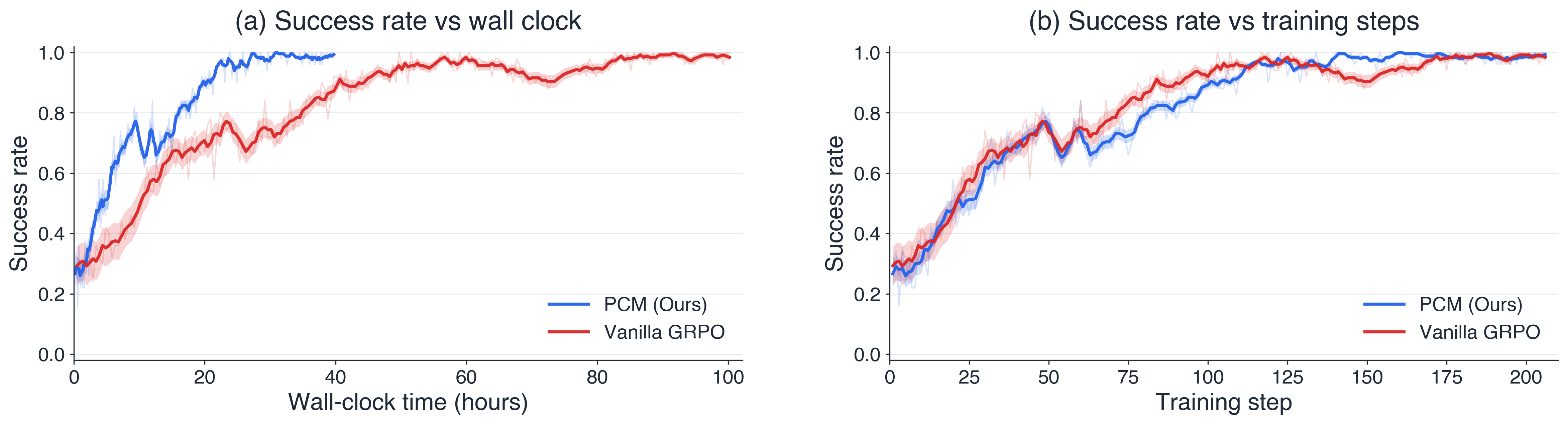}
    \vspace{-4pt}
    \caption{Success rate as a function of wall-clock time \textbf{(a)}
    and training step \textbf{(b)} on LIBERO-Object, averaged over three
    seeds. \textsc{PCM} ($B{=}12$) closely matches full-trajectory GRPO
    across training steps while converging significantly faster.}
    \label{fig:success-rate}
    \vspace{-4pt}
\end{figure}
We compare \textsc{PCM} with chunk budget $B{=}12$ against full-trajectory GRPO on LIBERO-Object using the same SFT-initialized OpenVLA-OFT checkpoint, identical rollout groups, rewards, advantages, and optimizer settings. The methods differ only in whether the actor update is computed over all chunks or the subset selected by PCM. Both are trained for a fixed budget of 200 steps.

As a function of training step (Fig.~\ref{fig:success-rate}b), PCM and GRPO exhibit nearly identical learning curves and reach matched final accuracy. This shows that PCM preserves per-step learning dynamics while improving efficiency, confirming that \textit{the selective gradient allocation does not degrade step-wise convergence behavior}.
\begin{figure}
    \centering
    \includegraphics[width=1\linewidth]{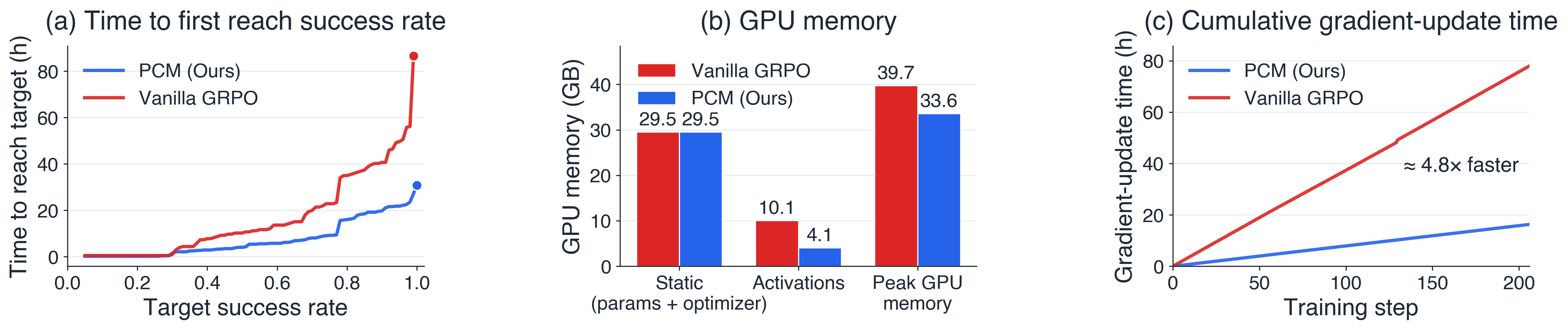}
    \vspace{-4pt}
    \caption{Efficiency on LIBERO-Object at $B{=}12$.
    \textbf{(a)} Wall-clock time to first reach each SR threshold; the
    \textsc{PCM}/GRPO gap widens at higher thresholds.
    \textbf{(b)} Activation memory is reduced by 60\%.
    \textbf{(c)} Cumulative actor-update time over 200 steps is
    $4.8\times$ faster.}
    \label{fig:metrics}
\end{figure}

\paragraph{RQ2: Wall-Clock Efficiency.} 

\begin{wraptable}{r}{0.58\linewidth}
\vspace{-8pt}
\centering
\small
\setlength{\tabcolsep}{4pt}
\begin{tabular}{lcc}
\toprule
 & Vanilla GRPO & PCM (Ours) \\
\midrule
LIBERO-Object  & $45.78 \pm 0.95$ & $\mathbf{19.23 \pm 0.57}$ \\
LIBERO-Goal    & $51.25 \pm 1.05$ & $\mathbf{21.18 \pm 0.59}$ \\
LIBERO-Spatial & $49.89 \pm 0.98$ & $\mathbf{21.23 \pm 0.63}$ \\
\midrule
Overall        & $48.97 \pm 0.99$ & $\mathbf{20.55 \pm 0.60}$ \\
\bottomrule
\end{tabular}
\vspace{3pt}
\caption{Wall-clock time in hours to reach $98\%\pm0.02$ success rate. Lower is better.}
\label{tab:convergence_comparison}
\vspace{-10pt}
\end{wraptable}

As shown in Fig.~\ref{fig:success-rate}a, PCM reaches near-saturation significantly earlier in wall-clock time. Fig.~\ref{fig:metrics}a shows that PCM reaches the target success rate of $98\%\pm0.02$ up to $2.38{\times}$ faster than vanilla GRPO, with the gap widening sharply at higher accuracy thresholds. Table~\ref{tab:convergence_comparison} shows that this trend holds across all LIBERO benchmarks, reducing the average time to $98\%\pm0.02$ success from $48.97\pm0.99$ to $20.55\pm0.60$ hours.
To understand how this speedup decomposes, Fig.~\ref{fig:metrics}b-c separate per-step compute from overall convergence. PCM reduces activation memory by $60\%$ ($10.1 \rightarrow 4.1$ GB) and peak GPU memory by $15\%$ ($39.7 \rightarrow 33.6$ GB) per step. 

Cumulative gradient-update time over 200 training steps is $4.8\times$ faster. Because the per-step learning curves are matched (Fig.~\ref{fig:success-rate}b), the wall-clock gain is entirely attributable to per-step compute savings rather than improved sample efficiency. Together these results show that the \textit{selective gradient allocation substantially reduces per-step cost while leaving sample efficiency intact}.


\paragraph{RQ3: Sensitivity to Chunk Budget $B$.} 


\begin{figure}[!h]
    \centering
    \includegraphics[width=1\linewidth]{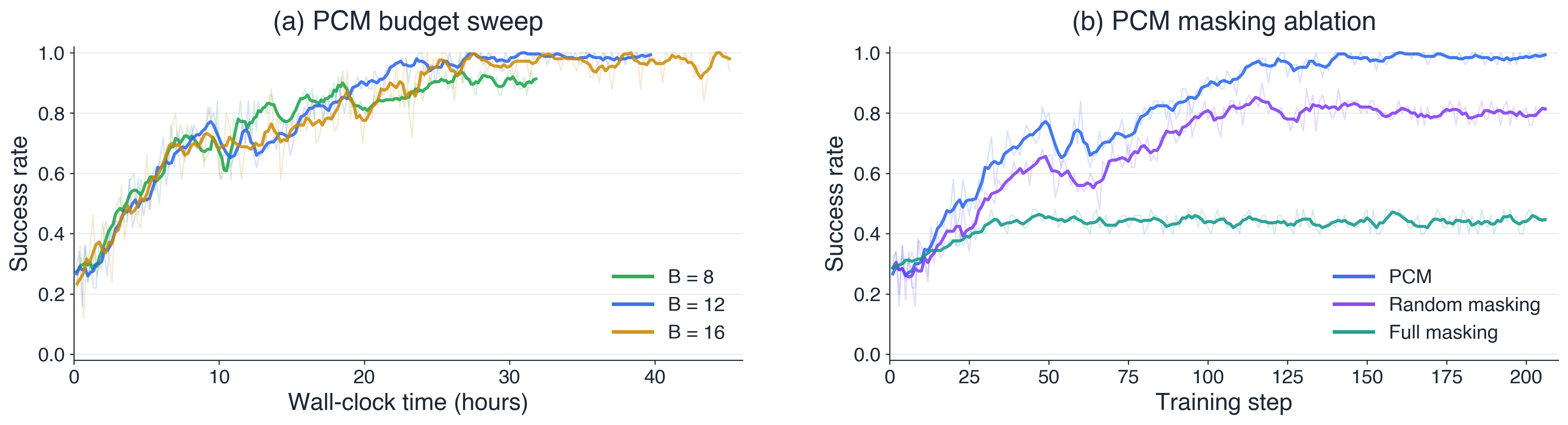}
    \vspace{-4pt}
    \caption{Ablations on LIBERO-Object.
    \textbf{(a)} Chunk budget sweep shows that $B{=}12$ preserves final
    success while retaining most of the wall-clock speedup.
    \textbf{(b)} At fixed $B{=}12$, \textsc{PCM} outperforms random
    masking and highest-variance-phase selection.}
    \label{fig:ablations}
\end{figure}

\textsc{PCM} samples a fixed budget of $B$ chunks per trajectory. Smaller budgets improve per-step efficiency but reduce gradient coverage, while larger budgets better approximate full GRPO at higher cost. We sweep B $\in$ \{8, 12, 16\} on LIBERO-Object using the same training configuration. We select this range for $B$ to bracket the inflection of the cumulative $C_c$ curve in Fig.~\ref{fig:chunk_budget}b, plotted against the fraction of trajectory chunks retained, which marks the smallest budget capturing most of the available learning signal. Details in Appendix~\ref{sec:chunk_budget}.

Fig.~\ref{fig:ablations}a shows that all budgets reach comparable final accuracy, but differ
in efficiency. $B{=}8$ is fastest per step but has degraded sample efficiency and plateaus at a slightly lower accuracy, indicating insufficient gradient signal. $B{=}16$ provides no accuracy gain over $B{=}12$ while increasing convergence time, showing diminishing returns beyond high-$C_c$ phases. $B{=}12$ lies at the inflection point, matching final accuracy while preserving the wall-clock advantage. We use $B{=}12$ as our default for all other experiments. $B{=}12$ is large enough to consistently sample from high-$C_c$ phases under the probabilistic selection rule while remaining small enough to deliver substantial per-step compute and memory savings.

\paragraph{RQ4: Variance Concentration and the Role of Variance-Aware Selection.} 

To test whether variance-aware probabilistic selection is necessary, we compare \textsc{PCM} at $B{=}12$ against two ablations that occupy opposite extremes of the allocation spectrum: \textit{random masking}, which samples $B$ chunks uniformly at random from each trajectory (no concentration on high-$C_c$ phases), and \textit{full masking}, which restricts updates to the single highest-$C_c$ phase and discards all other chunks (maximum concentration, no exploration). Together, these bracket \textsc{PCM}, which lies between these extremes via probabilistic $C_c$-weighted sampling.

Fig.~\ref{fig:ablations}b shows that both ablations underperform PCM. Random masking plateaus at $\sim$78\% success rate, 22 points below PCM at the same chunk budget. Full masking fails more sharply, plateauing at a low accuracy of ${\sim}43\%$. Since all methods use the same $B{=}12$ budget, these gaps isolate the contribution of variance-aware selection from compute reduction.

Fig.~\ref{fig:grad_prog} shows PCM's realized gradient allocation across training, revealing three mechanisms that together explain its performance.

\begin{wrapfigure}{r}{0.45\linewidth}
\vspace{-10pt}
\centering
\includegraphics[width=\linewidth]{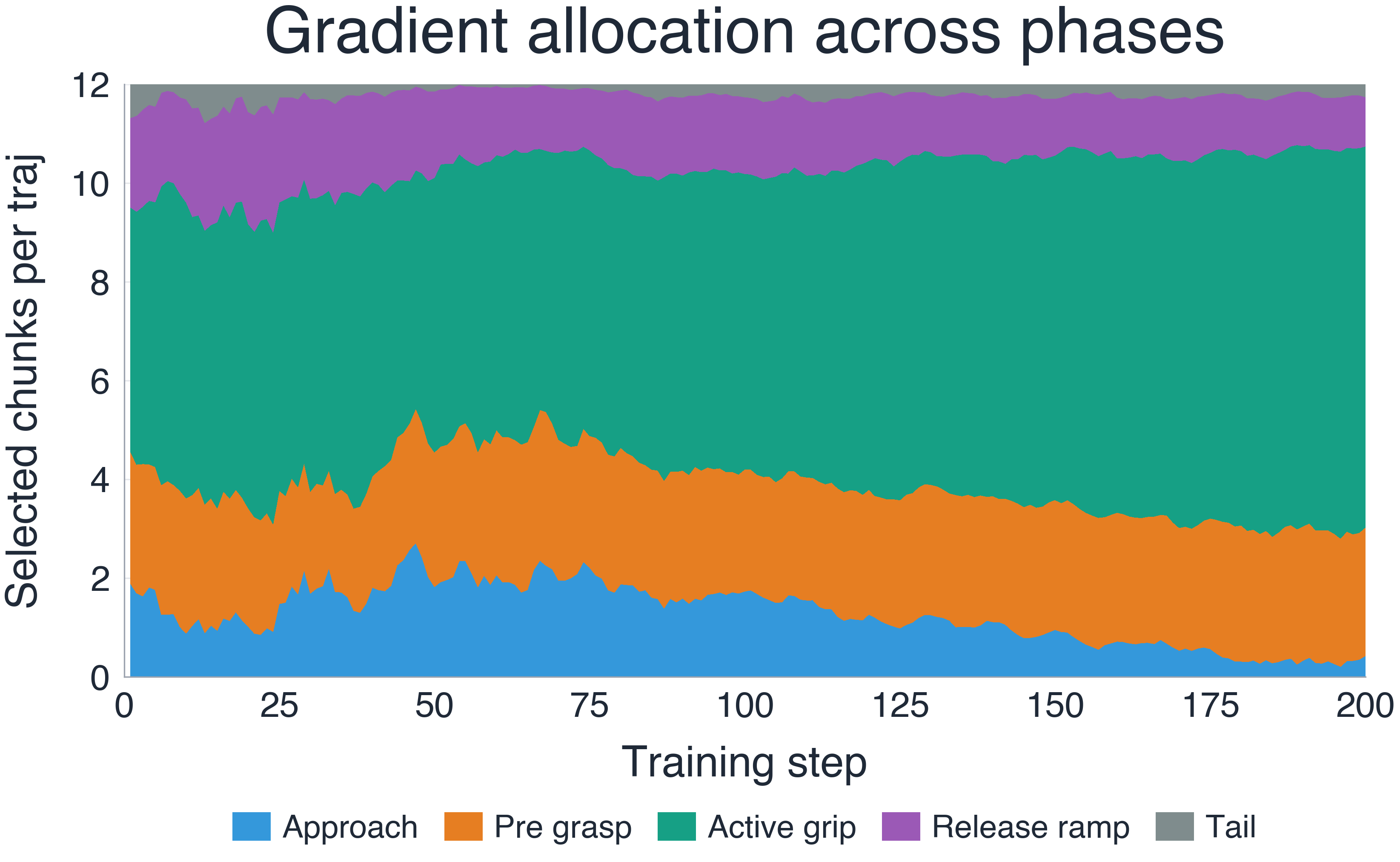}
\vspace{-10pt}
\caption{Phase-wise gradient allocation under PCM over training steps.}
\vspace{-12pt}
\label{fig:grad_prog}
\end{wrapfigure}

\noindent \textbullet\ \, \textbf{Concentration:} Active-grip ($\sim$7--8 chunks per trajectory) and pre-grasp
($\sim$3--5 chunks per trajectory) dominate throughout, matching the high-$C_c$ phases (Fig.~\ref{fig:phase_time}) and the $\sqrt{V_c}$-weighted allocation predicted by Theorem~\ref{thm:main}. The agreement between predicted and realized allocation is the empirical confirmation that $C_c$ tracks $V_c$ in practice, validating Lemma 2's proxy assumption.

\noindent \textbullet\ \, \textbf{Exploration:} 
Probabilistic $C_c$-weighted sampling (Eq.~\ref{eq:sample}) ensures that every phase has nonzero selection probability whenever its weight is above zero, providing the primary exploration mechanism. The $p_{\min} = 0.1$ floor (Eq.~\ref{eq:keepprob}) acts as a safety net for low-scoring phases
(e.g., \textit{tail}): it prevents transient zero estimates from collapsing a phase's weight to zero between buffer recomputes, which would otherwise cause online phase scores to fail to recover when the true $V_c$ ordering shifts during training.

\noindent \textbullet\ \, \textbf{Adaptation:} Approach and release-ramp
allocations decline from $\sim$2 to $\sim$1 chunks and from $\sim$2.5 to
$\sim$1.5 chunks, respectively, as the policy masters those phases, with budget shifting
toward phases where learning signal concentrates. This online dynamic that adapts with evolving success-failure gap of the policy cannot be reproduced by static weighting.

The ablations occupy degenerate corners of this design space. Random
masking explores but lacks concentration, matching the suboptimal
uniform-allocation regime in Theorem~\ref{thm:main}, where $b_c = B/K$ (equal budget across phases) is 
inefficient when $V_c$ is non-uniform across phases. Full masking concentrates updates on the highest-$C_c$ phase but discards weaker adjacent-phase signals and prevents exploration and adaptation.

Only PCM has all three properties: $C_c$-weighted keep
probabilities with a $p_{\min}$ floor to concentrate updates on
contact-rich phases while retaining lower-scoring phases. The ablations show that variance-aware probabilistic selection substantially outperforms uninformed masking strategies in exploiting the concentrated structure of $C_c$.

\section{Limitations and Future Work}

Our evaluation spans the three LIBERO suites (Object, Spatial, Goal), which test transfer of object, spatial, and task knowledge respectively (Table \ref{tab:convergence_comparison}). We do not test longer horizons or bimanual coordination; however, the underlying allocation principle applies whenever $V_c$ is non-uniform, a structural feature of any GRPO setting where the policy has not converged uniformly. Experiments use a gripper-based phase partition rule. Extension to other settings is straightforward in principle (any consistent decomposition suffices), though we do not validate alternative phase abstractions empirically. We do not directly benchmark against methods targeting sample efficiency through different objectives (e.g., entropy-based token selection or prompt filtering). These methods differ from PCM in both granularity (token vs. phase) and signal (policy-internal uncertainty vs. outcome-grounded divergence); adapting them to the VLA chunk-level setting is beyond the scope of this work. 

Future work includes extending PCM to longer-horizon and bimanual tasks, validating learned temporal phase segmenters or LLM-derived phase labelers, and applying the variance-allocation principle to LLM reasoning where $C_c$ can be computed against verified outcomes. Combining PCM with orthogonal sample-efficiency methods remains an open direction.

\newpage

{\small
\bibliographystyle{plainnat}
\bibliography{references}
}








\newpage
\appendix

{\Large\bfseries Appendix}

\section{Problem Formulation and Assumptions}
\label{app:problem}


For clarity, we explicitly restate the problem formulation and assumptions underlying PCM. While these are implicit in the main text, we formalize them here to precisely define the input setting, the chunk-selection objective, and the structural conditions required for the analysis.

\paragraph{Input.} A VLA policy $\pi_{\theta}^0$ initialized from a supervised fine-tuning checkpoint, a simulated environment $\mathcal{E}$, a distribution over tasks $\mathcal{T}$ with binary reward $r: \tau \to \{0,1\}$, and a group of $G$ rollouts $\{\tau_i\}_{i=1}^G$ where each trajectory decomposes into $N_i = \lceil T_i/L \rceil$ chunks of length $L$, and a fixed gradient compute budget $B \ll \sum_i N_i$.

\paragraph{Output.} A chunk selection policy $\mathcal{M}: \{\tau_i, r_i\}_{i=1}^G \to \{K_i \subseteq \{1, \ldots, N_i\}\}_{i=1}^G$ that minimizes
\begin{align}
    \sum_{i=1}^G \mathrm{Var}\left(\sum_{k \in K_i} A_i \cdot \nabla_\theta \log \pi_\theta(a_{i,k} \mid s_{i,k})\right) \label{eq:objective}
\end{align}
subject to $|K_i| = B$ for all $i$ (fixed compute budget), and $\mathcal{M}$ using only $\{\tau_i, r_i\}_{i=1}^G$ without access to auxiliary reward models or learned critics.

\paragraph{Assumptions.} We assume the following:
\begin{itemize}[leftmargin=*]
    \item \textbf{Binary reward:} $r: \tau \to \{0,1\}$ is trajectory-level, with each group containing both successes and failures ($0 < \mu_r < 1$); when all rollouts succeed or fail, advantages collapse and the variance signal is uninformative.
    \item \textbf{Fixed chunk length:} The policy produces actions in chunks of fixed length $L$, so the phase decomposition $\phi(i,k)$ is well-defined and consistent across trajectories.
    \item \textbf{Deterministic phase labeling:} $\phi(i,k) \in \mathcal{P}$ is computable from the gripper command trajectory alone, requiring no learned model.
    \item \textbf{Shared phase structure:} Successful and failed rollouts share the same phase partition $\mathcal{P}$, so per-phase action variance compares like with like across outcome groups.
    \item \textbf{SFT initialization:} $\pi_\theta^0$ is initialized from a supervised fine-tuning checkpoint, ensuring gradient variance concentrates on a small subset of outcome-critical phases.
\end{itemize}

\section{Theoretical Results}
\label{sec:formal_proofs}

\setcounter{theorem}{0}
\setcounter{lemma}{0}

In this section we provide proofs for the formal results in the main paper.

\begin{lemma}[Phase gradient variance]
\label{lem:variance}
Let $c \in \mathcal{P}$ be a phase in which $\pi_\theta$ has converged, i.e.,
the action distributions of successful and failed rollouts are identical in
phase $c$.
Then $\|g_c(\theta)\| \approx 0$ and $V_c \approx 0$, so gradient
samples from phase $c$ contribute negligible signal.
Conversely, for a phase $c$ in which the base policy is underspecified,
$V_c$ is large and gradient samples from phase $c$ are informative.
\end{lemma}

\begin{proof}
Since $A_i$ is zero-mean by construction (Eq.~\eqref{eq:advantage}), we have
$V_c = 0$ if and only if
$\pi_\theta(a_{i,k} \mid s_{i,k})$ is identical across successful and failed
rollouts for every chunk $(i,k)$ with $\varphi(i,k) = c$.
This holds exactly when outcomes are independent of the actions taken in phase
$c$. When these conditional action distributions match, positive- and negative-advantage score-function terms are drawn from the same distribution and cancel in expectation, leaving $\|g_c(\theta)\|\approx 0$. When they diverge, phase $c$ carries useful learning signal.
\end{proof}

\begin{lemma}[$C_c$ as a proxy for $V_c$]
\label{lem:proxy}
For a policy $\pi_\theta$ that is locally Gaussian with per-dimension variance
$\sigma_\pi^2$, the per-phase gradient variance satisfies
  $V_c \;\geq\; {C_c^2}/{4\sigma_\pi^2}$.
\end{lemma}

 
\begin{proof}
Let $\mu^+_c = \mathbb{E}[a_{i,k} \mid r_i=1,\, \varphi(i,k)=c]$ and
$\mu^-_c = \mathbb{E}[a_{i,k} \mid r_i=0,\, \varphi(i,k)=c]$.
For a Gaussian policy $\pi_\theta(\cdot \mid s) = \mathcal{N}(\mu_\theta(s), \sigma_\pi^2 I)$,
the score function is $\nabla_\theta \log \pi_\theta(a \mid s) =
(a - \mu_\theta(s))/\sigma_\pi^2 \cdot \nabla_\theta \mu_\theta(s)$.
Conditioning on the outcome group, the variance of the score is lower-bounded
by the variance of the mean action across groups.
Specifically, for any random variable $X$ and a binary conditioning event $E$,
$\mathrm{Var}(X) \geq \mathrm{Var}(\mathbb{E}[X \mid E])$ by the law of total
variance. Applying this with $X = A_i \cdot \nabla_{\theta} \log \pi_{\theta}(a_{i,k} \mid s_{i,k})$ and $E = \{r_i = 1\}$ versus $\{r_i = 0\}$, and using $|A_i| \leq 1/\varepsilon$ as a finite bound, yields the equation in the lemma after absorbing the gradient norm $\|\nabla_{\theta} \mu_{\theta}\|^2$ and balanced-group constants into $\sigma_\pi^2$. The bound preserves the relative ordering of $\{V_c\}$ across phases.
\end{proof}

\begin{theorem}[Optimal phase allocation]
\label{thm:main}
Let $b_c$ denote the number of chunks sampled from 
phase $c$, with $\sum_{c \in \mathcal{P}} b_c = B$, 
and let $N_c$ denote the expected number of chunks 
in phase $c$ per trajectory. The allocation minimizing the variance of the unbiased ratio estimator of the GRPO gradient
(Eq.~(\ref{eq:grpo})) subject to the budget constraint is
\begin{equation}
    b_c^* = B \cdot \frac{N_c\sqrt{V_c}}
    {\sum_{c' \in \mathcal{P}} N_{c'}\sqrt{V_{c'}}}.
    \label{eq:optimal-alloc}
\end{equation}
\end{theorem}

\begin{proof} We analyze the idealized stratified ratio estimator with deterministic per-phase budget bc; the masked loss in Eq.~(\ref{eq:pcm}) implements this allocation in expectation under the sampling rule of Eq.~(\ref{eq:sample}), at the cost of a finite-sample bias bounded by the contribution of low-Vc phases (Sec.~\ref{sec:shrink}). The GRPO gradient decomposes by phase as in 
Eq.~\eqref{eq:decomp}. PCM estimates the full phase gradient 
$g_c(\theta)$, which sums over all $N_c$ chunks 
in phase $c$, via the ratio estimator
\begin{equation}
    \hat{g}_c = \frac{N_c}{b_c}\sum_{k \in K_c}
    A_i \cdot \nabla_\theta \log \pi_\theta
    (a_{i,k} \mid s_{i,k}),
    \label{eq:ratio-estimator}
\end{equation}
where $K_c \subseteq \{k : \varphi(i,k) = c\}$ 
with $|K_c| = b_c$ chunks drawn uniformly without 
replacement from the $N_c$ available chunks in 
phase $c$. The $N_c/b_c$ factor scales up the 
sampled sum to account for the full phase size, 
making $\hat{g}_c$ unbiased for $g_c(\theta)$:
\begin{equation*}
    \mathbb{E}_{K_c}[\hat{g}_c] 
    = \frac{N_c}{b_c} \cdot b_c \cdot 
    \frac{g_c(\theta)}{N_c} = g_c(\theta).
\end{equation*}
The variance of $\hat{g}_c$ is
\begin{equation}
    \mathrm{Var}(\hat{g}_c) = 
    \frac{N_c^2}{b_c^2} \cdot b_c \cdot V_c
    = \frac{N_c^2 V_c}{b_c},
    \label{eq:ratio-var}
\end{equation}
where $V_c$ is the per-chunk variance within 
phase $c$  (Eq.~\eqref{eq:Vc}), and the $b_c$ in the 
numerator is the variance of a single-chunk 
sample, reduced by $b_c$ draws.

The total estimator variance across all phases is
\begin{equation}
    \sigma^2(\{b_c\}) = \sum_{c \in \mathcal{P}} 
    \frac{N_c^2 V_c}{b_c}.
    \label{eq:total-var}
\end{equation}
We minimize $\sigma^2(\{b_c\})$ subject to 
$\sum_{c \in \mathcal{P}} b_c = B$, $b_c \geq 0$. 
By the Cauchy-Schwarz inequality,
\begin{equation*}
    \left(\sum_{c \in \mathcal{P}} N_c\sqrt{V_c}
    \right)^2
    = \left(\sum_{c \in \mathcal{P}} 
      \frac{N_c\sqrt{V_c}}{\sqrt{b_c}} \cdot 
      \sqrt{b_c}\right)^2
    \leq 
    \underbrace{\left(\sum_{c \in \mathcal{P}} 
    \frac{N_c^2 V_c}{b_c}
    \right)}_{\sigma^2(\{b_c\})}
    \cdot
    \underbrace{\left(\sum_{c \in \mathcal{P}} 
    b_c\right)}_{B},
\end{equation*}
with equality if and only if 
$N_c\sqrt{V_c}/\sqrt{b_c} \propto \sqrt{b_c}$, 
i.e., $b_c \propto N_c\sqrt{V_c}$.
Combining with $\sum_c b_c = B$ yields 
Eq.~\eqref{eq:optimal-alloc}.

Substituting $b_c^* = B \cdot N_c\sqrt{V_c} / 
\sum_{c'} N_{c'}\sqrt{V_{c'}}$ into 
Eq.~\eqref{eq:total-var}, the minimum total 
variance achieved is
\begin{equation}
    \sigma^2_* = \frac{1}{B}
    \left(\sum_{c \in \mathcal{P}} 
    N_c\sqrt{V_c}\right)^2.
    \label{eq:min-var}
\end{equation}
\end{proof}

\subsection{Convergence Derivation for Phase Allocation}
\label{app:allocation}

We derive the convergence-rate comparison between 
the optimal allocation of 
Theorem~\ref{thm:main} and uniform 
allocation. Under standard SGD assumptions 
(L-smooth objective, bounded per-phase gradient 
variance $V_c$), the number of steps to reach 
an $\varepsilon$-stationary point 
satisfies~\citep{bottou2018optimization}
\begin{equation}
    T(\varepsilon) = \mathcal{O}\!\left(
    \frac{\sigma^2}{\varepsilon^2}
    \right),
    \label{eq:sgd-rate}
\end{equation}
where $\sigma^2$ is the variance of the gradient 
estimator at sample budget $B$; the $1/B$ variance reduction from $B$-sample averaging is incorporated into $\sigma^2$ via the per-phase variance formulas.

\paragraph{Optimal allocation.}
From Theorem~\ref{thm:main} and 
Eq.~\eqref{eq:min-var}, the minimum achievable 
estimator variance under budget $B$ is
\begin{equation}
    \sigma^2_* = \frac{1}{B}
    \left(\sum_{c \in \mathcal{P}} 
    N_c\sqrt{V_c}\right)^2.
    \label{eq:sigma-opt}
\end{equation}
Substituting into Eq.~\eqref{eq:sgd-rate}:
\begin{equation}
    T^*(\varepsilon) = \mathcal{O}\!\left(
    \frac{\left(\sum_{c \in \mathcal{P}} 
        N_c\sqrt{V_c}\right)^2}
    {B \cdot \varepsilon^2}
    \right).
    \label{eq:T-opt}
\end{equation}

\paragraph{Uniform allocation.}
Under $b_c = B/K$ for all $K$ phases, 
the total estimator variance is
\begin{equation}
    \sigma^2_{\mathrm{uniform}} 
    = \sum_{c \in \mathcal{P}} 
    \frac{N_c^2 V_c}{B/K}
    = \frac{K}{B} \sum_{c \in \mathcal{P}} 
    N_c^2 V_c.
    \label{eq:sigma-unif}
\end{equation}
Substituting into Eq.~\eqref{eq:sgd-rate}:
\begin{equation}
    T_{\mathrm{uniform}}(\varepsilon) 
    = \mathcal{O}\!\left(
    \frac{K \sum_{c \in \mathcal{P}} N_c^2 V_c}
    {B \cdot \varepsilon^2}
    \right).
    \label{eq:T-unif}
\end{equation}

\paragraph{Speedup ratio.}
Dividing Eq.~\eqref{eq:T-unif} by 
Eq.~\eqref{eq:T-opt}:
\begin{equation}
    \Delta 
    = \frac{T_{\mathrm{uniform}}(\varepsilon)}
           {T^*(\varepsilon)}
    = \frac{K \displaystyle
      \sum_{c \in \mathcal{P}} N_c^2 V_c}
    {\left(\displaystyle\sum_{c \in \mathcal{P}} 
      N_c\sqrt{V_c}\right)^2}.
    \label{eq:speedup}
\end{equation}
By the Cauchy-Schwarz inequality,
\begin{equation*}
    \left(\sum_{c \in \mathcal{P}} 
    N_c\sqrt{V_c}\right)^2 
    \leq K \sum_{c \in \mathcal{P}} N_c^2 V_c,
\end{equation*}
so $\Delta \geq 1$, with equality if and only 
if all $N_c^2 V_c$ are equal across phases. 
When variance concentrates on few phases such 
that $\max_c N_c^2 V_c \gg \frac{1}{K}
\sum_c N_c^2 V_c$, then $\Delta \gg 1$: the 
optimal allocation yields substantially fewer 
convergence steps than uniform allocation at 
the same chunk budget $B$.

\subsection{Bias Analysis of the Masked Estimator}
\label{app:bias}

The masked objective in Eq.~(11) omits the $N_c/b_c$ rescaling of the unbiased estimator (Eq.~14), trading a small bias for lower variance. We bound this bias.

\paragraph{Lemma 3 (Bias of the masked estimator).}
\emph{Let $p_c$ denote the keep probability for phase $c$. The bias of the masked estimator relative to the unbiased estimator satisfies}
\[
\|\mathrm{bias}\| = \Big\| \sum_{c \in \mathcal{P}} (1 - p_c)\, g_c \Big\| \;\leq\; \sum_{c \in \mathcal{P}} (1 - p_c)\, \|g_c\|.
\]

\textit{Proof.} Under weighted sampling without replacement with inclusion probability $p_c$, the expected contribution of phase $c$ to the masked estimator is $p_c \cdot g_c$, while the unbiased estimator yields $g_c$. Summing across phases and applying the triangle inequality gives the bound. \qed

Under the allocation $p_c \propto \sqrt{V_c}$, $(1-p_c)$ is large only when $V_c$ is small, and by Lemma~1, small $V_c$ implies $\|g_c\| \to 0$. Each term in the bound is therefore suppressed by at least one small factor, so Theorem~1's allocation remains approximately optimal under the biased estimator.

\section{Methods We Tried That Did Not Work}

The motivating observation behind this paper was that VLA rollouts on a given task are remarkably similar across trajectories. Most timesteps execute nearly identical motion patterns regardless of whether the trajectory eventually succeeds. A reaching phase looks like a reaching phase across all rollouts; the divergence concentrates on a small number of decision points. This led us to hypothesize that gradient compute was being misallocated: the same per-step budget was being spent on phases that contributed little to the policy update as on phases that determined the outcome. Before arriving at PCM, we explored two alternative directions for exploiting this observation. Both failed, and the failures shaped the eventual design.

\subsection{Branching at Decision Critical Timesteps}

Our initial approach attempted to increase the rollout signal at high-uncertainty points by branching: at each candidate decision point, fork the rollout into K alternative continuations to give GRPO more contrastive signal in the regions that mattered. Two problems made this infeasible. First, each branch required actually executing a downstream rollout in the simulator to be useful for GRPO, which meant returning to the branch point and continuing the simulation; this added significant rollout-side overhead and partially undid the efficiency motivation. Second, naive branching is exponential in depth, refining beyond the first branch point requires nested branching, and the rollout count grows as $K^d$ where d is the number of branches per trajectory. We found no principled way to bound depth without re-introducing the same uniform-allocation problem we were trying to escape.

\begin{figure}[!h]
    \centering
    \includegraphics[width=1\linewidth]{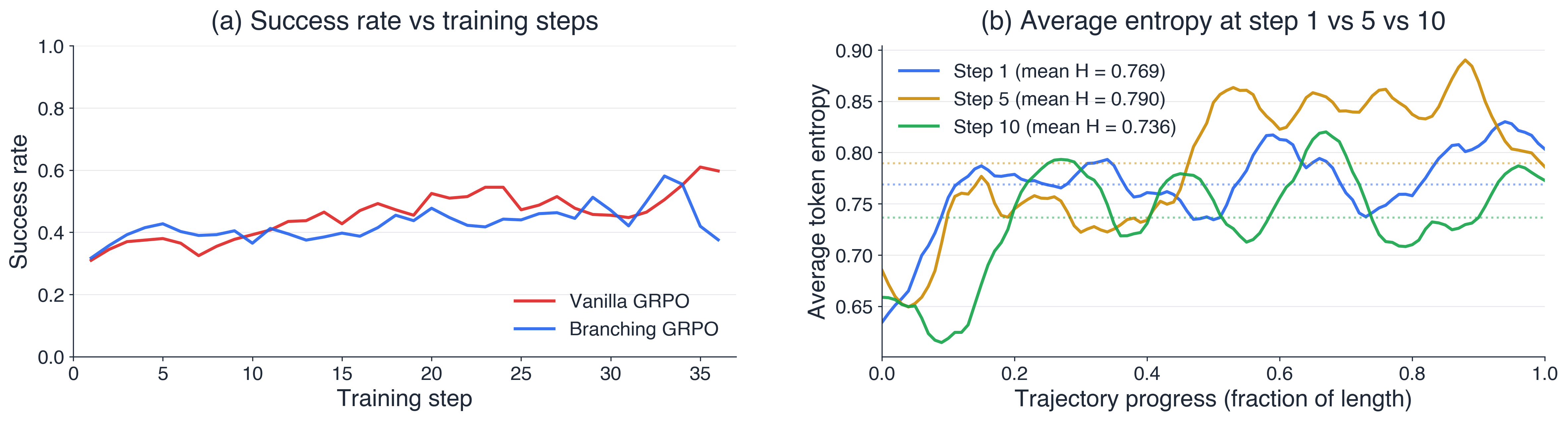}
    \caption{(a) Success vs. training steps for vanilla GRPO vs branching GRPO. (b) Action entropy over trajectory progress at different stages of training, showing weak alignment with outcome-critical phases. The training and evaluation follow the same procedures as Sec.~\ref{sec:results}}.
    \label{fig:placeholder}
\end{figure}

The lesson from this attempt was that any practical method had to operate within the existing rollout structure rather than expanding it. PCM instead reduces gradient computation over rollouts that have already been generated.

\subsection{Entropy as a Signal to Concentrate Learning}

Our second attempt used per-chunk policy entropy as a proxy for "where the model has room to learn." The hypothesis was the standard one: high-entropy chunks indicate uncertainty, uncertainty indicates a learning opportunity, and concentrating gradient on these chunks should yield faster convergence. We expected that as training progressed, entropy would decrease in initially-uncertain chunks and the gradient allocation would naturally shift toward the remaining noisy regions.

Two empirical findings led us to abandon this approach. First, entropy was poorly structured across trajectories: we did not observe consistent concentration on decision-critical phases the way we eventually observed for $C_c$ (Figure 1). Per-chunk entropy varied substantially across rollouts even within the same task and phase, and the high-entropy chunks did not align with the points where successful and failed rollouts actually diverged. We attribute this to the fact that VLA action distributions are typically already low-entropy after SFT. The supervised pretraining drives the policy toward confident, low-entropy actions, and the residual entropy reflects modeling noise rather than genuine multimodal uncertainty over correct behavior.
Second, and more decisively, entropy did not meaningfully decrease over training. Across multiple runs on LIBERO, we measured per-chunk entropy at the start and end of GRPO fine-tuning and found negligible reduction. This is consistent with the known property that GRPO is a relatively weak signal for changing the underlying action distribution: the group-relative advantage primarily reweights existing modes rather than collapsing the policy onto a sharper one. Entropy under SFT-initialized policies is therefore close to a fixed property of the model rather than a quantity RL can drive down. Selecting on entropy would mean repeatedly allocating gradient to the same chunks throughout training regardless of whether they had become well-learned, and would not adapt to the shifting locus of the success–failure gap as training progressed.

\section{Analysis of Chunk Budget}
\label{sec:chunk_budget}

\begin{figure}
    \centering
    \includegraphics[width=0.6\linewidth]{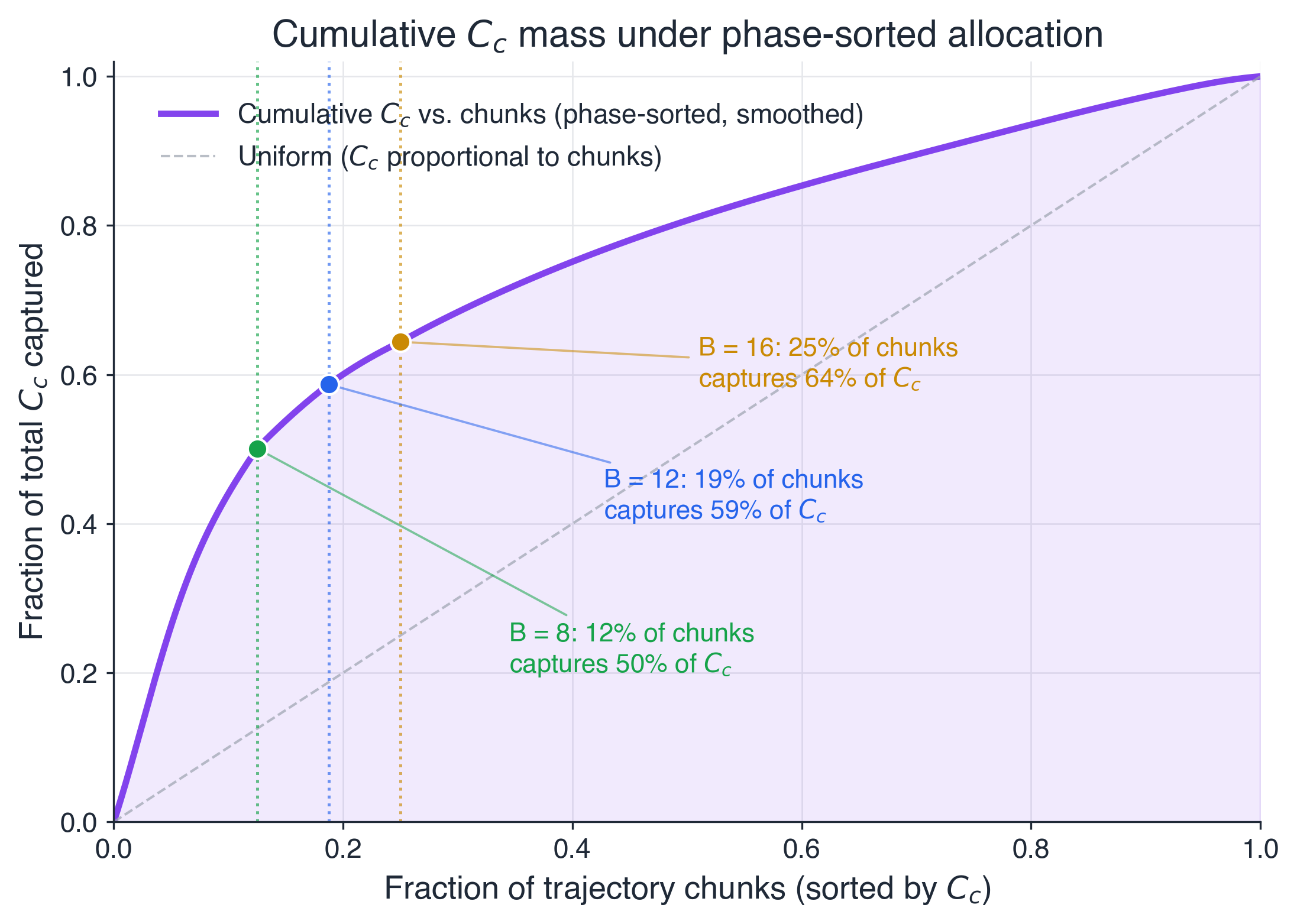}
    \caption{
    Cumulative $C_c$ captured as a function of the fraction of trajectory chunks retained, with chunks sorted in descending order by phase-level $C_c$. The solid curve shows the empirical cumulative $C_c$, while the dashed diagonal shows the uniform baseline under which $C_c$ would scale linearly with chunk count. The gap between the two reflects the uneven distribution of learning signal across the trajectory: a small fraction of chunks captures a disproportionate share of the total $C_c$.}
    \label{fig:chunk_budget}
\end{figure}

\paragraph{The compute-signal trade-off.}

Choosing $B$ reflects a trade-off between two opposing pressures. Lowering $B$ reduces per-step gradient computation linearly: fewer chunks in the backward pass reduce activation memory and accelerate actor updates, which is the source of PCM's wall-clock advantage. Increasing $B$ captures more of the available learning signal: a larger budget includes more chunks from outcome-divergent phases and approaches the gradient quality of full-trajectory GRPO.

The cumulative $C_c$ curve in Fig.~\ref{fig:chunk_budget} is sharply concave: early chunks, corresponding to high-$C_c$ phases, deliver disproportionate signal, while additional chunks beyond the knee contribute progressively less. The \emph{knee} of this curve, the point at which marginal $C_c$ capture per chunk drops sharply, occurs at approximately $20\%$ of trajectory chunks. This concavity makes the trade-off non-trivial. Below the knee, each additional chunk yields substantial signal at fixed compute cost, so excluding it sacrifices accuracy for marginal speedup. Above the knee, each additional chunk yields diminishing signal at the same cost, so including it sacrifices speedup for marginal accuracy. The knee marks the point where these pressures balance: the smallest $B$ for which captured $C_c$ has largely saturated and further increases primarily incur additional compute cost.

\paragraph{Budget range selection.}

The cumulative $C_c$ curve in Fig.~\ref{fig:chunk_budget} is sharply concave, with a knee at approximately $20\%$ of trajectory chunks. We
compute captured $C_c$ by ranking chunks according to their phase score
and measuring the fraction of total chunk-level score mass retained by
the top $B$ chunks. Beyond the knee, each additional chunk yields progressively smaller marginal gains in captured $C_c$. We select $B \in \{8, 12, 16\}$ to bracket this knee, corresponding to $12\%$, $19\%$, and $25\%$ of trajectory chunks, respectively.

\begin{itemize}
    \item \textbf{$B = 8$} ($12\%$ of chunks, $50\%$ of $C_c$ captured): Below the knee. Tests whether aggressive truncation that falls short of the inflection leads to signal under-coverage and degraded final accuracy.
    
    \item \textbf{$B = 12$} ($19\%$ of chunks, $59\%$ of $C_c$ captured): At the knee. Tests the predicted operating point where $C_c$ capture is high and marginal returns begin to flatten.
    
    \item \textbf{$B = 16$} ($25\%$ of chunks, $64\%$ of $C_c$ captured): Beyond the knee. Tests whether additional budget yields meaningful gains. The extra $5$ percentage points of $C_c$ capture over $B = 12$ come at substantially higher per-step compute.
\end{itemize}

The empirical sweep in Sec.~\ref{sec:results} (RQ3) confirms this curvature: $B = 8$ underperforms in final accuracy, $B = 16$ matches $B = 12$ but at higher cost, and $B = 12$ emerges as the operating point consistent with the knee of the cumulative $C_c$ curve.

\paragraph{Choosing $B$ for other domains.}

The knee of the cumulative $C_c$ curve provides a principled approximation for $B$ in any RL setting with a defined phase partition. The reasoning generalizes directly from the compute--signal trade-off: $C_c$ is computed from rollouts that GRPO already produces, so the curve can be measured in any domain where outcomes are verifiable and rollouts decompose into phases. The structural assumptions of the method are therefore not specific to manipulation. The knee identifies the operating point at which the two pressures balance, independent of the absolute magnitude of $C_c$ or the number of phases, because it is determined by the curvature of the signal.

The procedure is the same across domains. For a target domain, compute $C_c$ per phase from a single rollout group, plot the cumulative $C_c$ curve as in Fig.~\ref{fig:chunk_budget}, and select $B$ at the knee, the smallest fraction of chunks at which captured $C_c$ has largely saturated. The phase partition itself is domain-specific, such as gripper-state heuristics for manipulation, semantic decomposition for LLM reasoning, or episode segmentation for dialogue. Once defined, the curve and its knee are computed identically. Domains with sharper curvature, where signal concentrates in a few phases, admit smaller $B$ and larger speedups. Flatter curves indicate more uniform signal and require larger $B$ to capture comparable $C_c$, yielding smaller speedups. Only the resulting value of $B$ varies; the selection rule remains unchanged.

\section{Implementation Details}
\label{app:setup}

This appendix provides additional experimental details omitted from the
main text, including benchmark construction, training hyperparameters,
\textsc{PCM}-specific settings, and evaluation protocol. Unless stated
otherwise, \textsc{PCM} and the full-trajectory GRPO baseline use the
same rollout, reward, advantage, optimizer, and evaluation configuration.

\subsection{Benchmarks}
\label{app:benchmarks}
\paragraph{Models and Benchmarks.}
We evaluate on three LIBERO suites to cover diverse
manipulation regimes and test generalization. LIBERO-Object tests
object-centric pick-and-place, LIBERO-Spatial tests spatial-relation following,
and LIBERO-Goal tests goal-conditioned manipulation. The manipulation tasks are specified by natural-language instructions.
The following examples illustrate the types of prompts used across the
benchmarks:

\begin{itemize}
    \item Object-centric manipulation: ``pick up the mug and place it in the bowl.''
    \item Spatial manipulation: ``place the object on the left side of the tray.''
    \item Goal-conditioned manipulation: ``move the object to the target location.''
\end{itemize}
The policy
receives the language instruction and RGB observation and predicts a sequence
of 7-DoF end-effector actions. These suites vary in object identity, spatial layout, goal specification, and
task structure, bringing diversity in manipulation regimes and causing
phase-level failure patterns to appear with different frequencies across
benchmarks. Each benchmark contains 10 tasks with 50 trials per task under different initial object configurations, yielding 500 evaluation trials per benchmark. Trajectories contain a total of 64 chunks. During RL fine-tuning, we use visual augmentations such as brightness changes and image jitter to improve
robustness.

\paragraph{Training.}
Our method is implemented based on the SimpleVLA-RL \texttt{verl}
pipeline~\citep{10.1145/3689031.3696075}. We initialize from the SimpleVLA-RL 1-trajectory SFT checkpoint, where
OpenVLA-OFT is supervised-fine-tuned with one demonstration per task. We
then perform RL fine-tuning with LoRA rank $r{=}32$ and $\alpha{=}32$ on
2 NVIDIA H100 GPUs.
We use a prompt batch size of 8 and sample 10
rollouts per prompt, giving 80 trajectories per training step. The
trajectory mini-batch size is 4, and the PPO mini-batch size is 2 with
micro-batch size 1. We optimize the actor with AdamW using a constant learning rate of
$1\times 10^{-5}$, $(\beta_1,\beta_2)=(0.9,0.999)$, weight decay $0$,
and gradient clipping at $2.0$. Following SimpleVLA-RL, we use one PPO
epoch per update, asymmetric clipping
$(\epsilon_{\mathrm{low}},\epsilon_{\mathrm{high}})=(0.2,0.4)$, entropy
coefficient $10^{-3}$, and disable the KL penalty. Rollouts are sampled
with temperature $1.6$, top-$p=1.0$, and top-$k=-1$ without truncation.

For \textsc{PCM}, we use a fixed chunk budget of $B{=}12$ chunks per
trajectory. Phase keep probabilities are initialized from the
success--failure phase scores $C_c$ computed from the rollout group, and
then recomputed every $T_{\mathrm{rc}}{=}5$ training steps from the
recent phase-score buffer. We clip keep probabilities below by
$p_{\min}{=}0.1$, ensuring that low-variance phases are not completely
discarded and preserving a small amount of exploration throughout
training. Non-selected chunks are physically removed before the actor
forward and backward pass. \textsc{PCM} is simple to integrate into existing GRPO pipelines: it only adds phase-score computation from rollout
statistics and a chunk-shrinking step before the actor update. We use the
same \textsc{PCM} hyperparameters across all benchmarks without
task-specific tuning.



\subsection{Phase Allotment Rules}
\label{app:prompts}
We assign trajectory chunks to semantic phases using a deterministic
gripper-based heuristic. For each chunk $j$, we compute $g_f[j]\in[0,1]$,
the fraction of timesteps in the chunk where the gripper-close command is
active. We use a sustained-close threshold $\tau=0.75$ to locate grasp
intervals, and then assign phase labels around those intervals.

\begin{itemize}
\item \textbf{Active-grip.} Chunks with substantial gripper closure
    are labeled \textit{active-grip}. We use a softer threshold
    $g_f[j]\geq 0.5$ for labeling, so this phase includes both sustained
    grasp chunks and nearby partial-close chunks. In successful rollouts,
    this covers grasp and transport while the object is held, or recovery after drop; in failed
    rollouts, it can also capture drops, re-grasp attempts, or repeated
    failed closure near the object.

    \item \textbf{Pre-grasp.} Up to three chunks immediately before a
    sustained-close interval are labeled \textit{pre-grasp} when the
    gripper has begun closing but has not yet reached sustained closure
    ($0.1 \leq g_f[j] < 0.5$). This phase often captures final alignment and
    orientation adjustment before contact.

    \item \textbf{Release-ramp.} Up to three chunks immediately after a
    sustained-close interval are labeled \textit{release-ramp} ($g_f[j]<0.5$), capturing
    the transition from holding the object to releasing it. This phase is
    important when failures occur because the object is released at the
    wrong location or does not fall into the target bin.

    \item \textbf{Approach.} Chunks before the pre-grasp/active-grip
    region of a grasp cycle are labeled \textit{approach} ($g_f[j]<0.5$), corresponding
    to navigation toward the object.

    \item \textbf{Tail} Chunks after the final release
    with an open gripper are labeled \textit{tail} ($g_f[j]<0.1$). These often capture
    post-release behavior, including cases where the object was not
    successfully placed and the policy continues hovering or moving after
    release or attempts.
\end{itemize}

The unified phase set is resolved by priority:
\[
\textit{active-grip} \;>\; \textit{pre-grasp} \;>\;
\textit{release-ramp} \;>\; \textit{approach} \;>\;
\textit{tail}
\]
When multiple phase rules apply due to overlapping windows, we assign the
highest-priority label. Phases are distinguished first by their
position relative to sustained-close intervals and then by the gripper
thresholds above: open-gripper chunks before a future grasp attempt are
\textit{approach}, whereas open-gripper chunks after the final
release-ramp window (after a sustained grasp interval) are \textit{tail}. Since labels are assigned at the chunk level, averaging the gripper command
over the chunk reduces sensitivity to single-timestep noise and makes the
thresholding rule more stable. 

The same rule is applied to successful and failed trajectories. The
terminal success label is never used by the phase classifier, so $C_c$
compares success and failure trajectories under a shared phase partition.
The labeling rule is lightweight, requires no learned phase model, and adds no
inference overhead. Other phase allocators, such as task-specific heuristics, temporal
semantic phase segmenters, or LLM-based labelers, could be used when
gripper state is insufficient.
\begin{figure}
    \centering
    \includegraphics[width=0.9\linewidth]{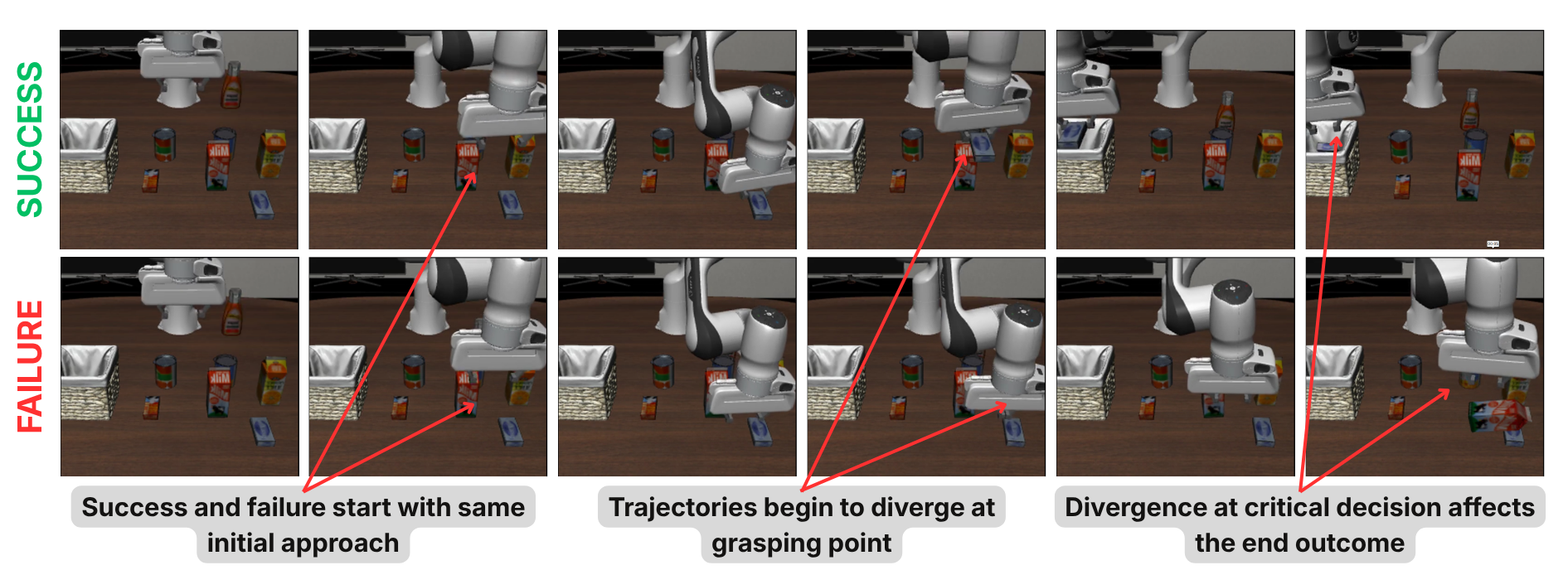}
    \caption{Successful and failed manipulation rollouts look similar for most
    of the trajectory, diverging mainly around outcome-critical grasp phases.}
    \label{fig:trajectories}
\end{figure}

\paragraph{Phase-based failure localization:}
Figure~\ref{fig:trajectories} illustrates how common failures fall
naturally into this phase partition. In a missed-grasp failure, the robot
may approach the object similarly to successful rollouts but close with
poor alignment. The closing chunks are labeled \textit{active-grip}, while
the immediately preceding orientation-adjustment chunks are labeled
\textit{pre-grasp}. This explains why both pre-grasp and active-grip can
receive high $C_c$: the failure may be caused not only by the closure
itself, but also by the pose and alignment right before contact. After such failures, however, many
remaining chunks are downstream consequence states (hovering, searching,
colliding, or moving without the object) which are typically labeled
\textit{tail}. These chunks can occupy a large
fraction of failed rollouts, but they are less causally informative than
the contact error that produced them. PCM therefore keeps them with low
nonzero probability while concentrating budget on the grasp phases that
prevent the policy from entering these failed states.

Other failures are captured by the same heuristic without additional case
rules. If the object is grasped but later dropped or repeatedly re-grasped
during transport, the corresponding close/hold chunks remain
\textit{active-grip}, making this phase sensitive to transport instability.
These chunks provide useful RL signal by penalizing slips, failed recovery and unstable holds
while rewarding successful recovery. If the grasp succeeds but placement fails, the error appears
around \textit{release-ramp}: the robot releases at the
wrong location or the object misses the bin, making its $C_c$ moderately high. Subsequent \textit{tail} chunks mostly capture open-gripper hovering after
release, and are therefore less outcome-discriminative and provide little learning signal. When computing $C_c$, we compare only phases that
contain both successful and failed samples in the rollout group; phases
absent from one outcome group in a batch are skipped for that update. In the observed rollout groups, failures often diverge during active-grip,
making it the most outcome-discriminative phase and causing $C_c$ to become
largest there. This concentration emerges automatically from the gripper-based
phase labels and rollout outcomes, without hand-designed phase weights or task-specific tuning.

\end{document}